\documentclass[a4paper,12pt]{article}

\usepackage[utf8]{inputenc}
\usepackage{algorithm}
\usepackage{algorithmicx}
\usepackage{algpseudocode}
\usepackage{amsmath}
\usepackage{amsthm,amssymb}
\usepackage{mathtools}
\usepackage{pgfplots}
\pgfplotsset{compat=1.14}

\allowdisplaybreaks

\newcommand{\hamm}[1]{\mathcal{H}(#1)}
\newcommand{\tobinary}[1]{\mathcal{B}(#1)}

\newcommand{\theop}{\mathcal{X}}

\newtheorem{theorem}{Theorem}
\newtheorem{lemma}{Lemma}

\newtheorem{definition}{Definition}

\begin{document}

\title{Black-Box Complexity of the Binary Value Function\footnote{An extended two-page abstract of this work will appear in proceedings of the Genetic and Evolutionary Computation
Conference, GECCO'19.}}
\author{Nina Bulanova, Maxim Buzdalov}

\maketitle

\begin{abstract}
The binary value function, or \textsc{BinVal}, has appeared in several studies in theory of evolutionary computation
as one of the extreme examples of linear pseudo-Boolean functions. Its unbiased black-box complexity was previously
shown to be at most $\lceil \log_2 n \rceil + 2$, where $n$ is the problem size.

We augment it with an upper bound of $\log_2 n + 2.42141558 - o(1)$, which is more precise for many values of $n$.
We also present a lower bound of $\log_2 n + 1.1186406 - o(1)$. Additionally, we prove that \textsc{BinVal} is an easiest
function among all unimodal pseudo-Boolean functions at least for unbiased algorithms.
\end{abstract}

\section{Introduction}
Theory of randomized search heuristics studies how various problems are solved by these heuristics,
whether it is efficient or not, and what are the key properties of both problems and heuristics that determine the
efficiency of the search and the quality of the results. In the current state of this area of computer science
there are two major building blocks that augment each other: runtime analysis and black-box complexity theory.
The former studies how fast particular randomized search heuristics are on particular problems or problem classes,
the latter strives to find how difficult it is to solve a problem (typically from the given class) by the best
suitable randomized search heuristic (and why). The gaps between the complexities of various problems and
the runtimes of existing algorithms on these problems are an important source of difficult questions and new inspiring results.

The (unrestricted) black-box complexity of a problem, as introduced in~\cite{droste-jansen-wegener},
is, roughly speaking, the expected runtime of the best possible black-box algorithm solving this problem, measured as the number of queries
to the function representing the quality of a solution to the problem.
Since every possible black-box algorithm is explicitly permitted, such complexity can be inadequate to explain the performance
of randomized search heuristics, as they typically use only a restricted set of possible operations with the candidate solutions.
By limiting the possible operations to some well-defined ``fair'' subset, researchers define more restricted notions of black-box complexity,
which hopefully better describe why (particular) randomized search heuristics are good or bad for certain problems.

In particular, the notion of unbiased black-box complexity was introduced in~\cite{unbiased-bbc-algorithmica} for pseudo-Boolean problems.
Since evolutionary algorithms and other randomized search heuristics are designed as
general-purpose solvers, they shall not prefer one instance of a problem over another one.
This is reflected in the definition of an unbiased black-box algorithm, which, for the particular case of
algorithms operating on bit strings, basically amounts to the invariance of the algorithm under transformations preserving
the Hamming distance between two candidate solutions, which reduces to invariance under systematic flipping of arbitrary but fixed set of bit indices,
and invariance under systematically applying an arbitrary but fixed permutation to all the bits.
Unbiased black-box algorithms are regarded as a better model of randomized search heuristics,
since these algorithms, just like randomized search heuristics, cannot perform a number of fine-grained operations
that can be considered problem-dependent.

Unfortunately, the ways were found to perform most of the work without making queries
in the unbiased model too~\cite{unbiased-partition-is-polynomial,too-fast-unbiased-bbc}.
In fact, it was shown that, with a proper notion of unbiasedness for the given type of individuals,
the unbiased black-box complexity coincides with the unrestricted one~\cite{generic-unbiased-algorithms}.
Several alternative restricted models of black-box algorithms were subsequently introduced
as a reaction, namely ranking-based algorithms~\cite{ranking-based-complexity},
limited-memory algorithms~\cite{mastermind-constant-memory},
and elitist algorithms~\cite{elitist-complexity}.

One of possible refinements of the unbiased black-box search model is the use of unbiased operators
with restricted arity. The original paper~\cite{unbiased-bbc-algorithmica} studied mostly \emph{unary} unbiased black-box complexity,
e.g.~the class of algorithms allowing only unbiased operators taking one individual and producing another
one, which can also be seen as \emph{mutation-only} algorithms. This model appeared to be quite restrictive, e.g.~the unary unbiased
black-box complexity of \textsc{OneMax} was proven to be $\Theta(n \log n)$~\cite{doerr-doerr-yang-optimal-parameter-choices-gecco,unbiased-bbc-algorithmica}.
Together with the rather old question of whether crossover is useful in evolutionary algorithms
(which was previously answered positively~\cite{doerr-all-pairs,jansen-crossover,sudholt-ising-models}),
this inspired a number of works on higher-arity unbiased algorithms,
since many crossovers are binary unbiased operators.

The theorem that for $k \ge 2$ the
$k$-ary unbiased black-box complexity of \textsc{OneMax} is $O(n / \log k)$, which was proven
in~\cite{doerr-johannsen-faster-blackbox}, was the first signal that higher arities are useful,
which have been followed by a more strong result of $O(n / k)$ for $k = O(\log n)$~\cite{doerr2014reducing}.
Among others, an elegant crossover-based algorithm with the expected running time of $2n - O(1)$
was presented, which works for all linear functions. Several particular properties of this algorithm
inspired the researchers to look deeper for faster general-purpose algorithms that use crossover.
The first reported progress of algorithms using crossover on simple problems like \textsc{OneMax}
was made in~\cite{sudholt-crossover-speeds-up-evco}, where an algorithm was presented
with the same $O(n \log n)$ asymptotic as in simple unary algorithms, but with a better constant factor.
An algorithm, called the $(1 + (\lambda,\lambda))$ genetic algorithm, was
presented in~\cite{learning-from-black-box-thcs} along with the proof of the $O(n \sqrt{\log n})$ runtime
on \textsc{OneMax}, which was faster than any evolutionary algorithm before,
and improved performance on some other problems.
With the use of self-adaptation for its parameter $\lambda$, the $O(n)$ bound was proven for the runtime
on \textsc{OneMax}~\cite{doerr-doerr-lambda-lambda-self-adjustment}, and similar improvements were shown
later on a more realistic problem, MAX-SAT~\cite{buzdalovD-gecco17-3cnf}.
Experiments show that the constant in $O(n)$ is rather small.

Another proof of the usefulness of the crossover has been recently presented in~\cite{pinto-doerr-crossover-ppsn18}.
The paper considers the $(2+1)$ genetic algorithm from~\cite{sudholt-crossover-speeds-up-evco}, which has been
previously shown to be faster than the $(1+1)$ evolutionary algorithm, but slower than randomized local search (RLS), on \textsc{OneMax}.
However, once the mutation operator is improved in such a way that, when it is the only operator to apply, it always flips at least one bit,
this algorithm becomes faster, by a constant factor, than not only RLS, but any other unary unbiased algorithm.
This result shows the usefulness of the crossover with a much simpler approach compared to~\cite{learning-from-black-box-thcs}.

Unfortunately, there are still no matching lower bounds known for $k$-ary unbiased black-box algorithms with $k > 1$, even for \textsc{OneMax}.
It seems to be believed that the binary unbiased black-complexity of \textsc{OneMax} is linear,
but no lower bounds, other than the trivial $\Omega(n / \log n)$ bound, are known.
A part of the difficulty of this problem is that, apart from the best individual like in the unary case~\cite{doerr-doerr-yang-optimal-parameter-choices-gecco},
one has to track at least one other individual, which need not be second best, to get the best from the ability to perform crossover, which complicates the possible proofs.

Our long-standing conjecture is that the binary unbiased black-box complexity is linear not only for \textsc{OneMax}, but for any linear function,
as the binary algorithm from~\cite{doerr-johannsen-faster-blackbox} works on linear functions without any changes. The class of linear functions
includes the \textsc{BinVal} function, a linear function with weights equal to powers of two.
It is notoriously known to have a ridiculously low complexity of $2 - \frac{1}{2^n}$ when the weights are known to be ordered.
A more relevant version that allows arbitrary permutation of weights as well as arbitrary bit strings as optima has been
considered in~\cite{ranking-based-complexity}, where the upper bound on its black-box complexity was proven to be $\lceil \log_2 n \rceil + 2$.
We feel that this is a more suitable function to study lower bounds on the unbiased binary black-box complexity
for the following reasons.
\begin{itemize}
    \item Any lower bound can be immediately extended not only to arbitrary linear functions, including \textsc{OneMax}, but to any class of unimodal functions,
          as any unbiased binary black-box algorithm that can optimize these functions can be applied to \textsc{BinVal}.
    \item The \textsc{BinVal} function reveals, through the fitness values, as much information about the structure of already sampled individuals
          as it is possible in an unbiased setting, which would simplify the analysis of lower bounds.
\end{itemize}

Although this paper does not contribute to the field of binary unbiased black-box complexities,
we augment the current knowledge about the \textsc{BinVal} function. This paper presents the following contribution.
\begin{itemize}
    \item A constructive \textbf{procedure} to compute the exact black-box complexity of \textsc{BinVal} for any given problem size $n$, which yields
          an \textbf{algorithm} to solve \textsc{BinVal} with the expected running time equal to this complexity.
    \item An \textbf{upper bound} on the black-box complexity of \textsc{BinVal}, which refines the known $\lceil \log_2 n \rceil + 2$ upper bound
          for many values of $n$ and is equal to $\log_2 n + 2.42141558 - \Theta(\log_2 n / 2^n)$.
    \item A \textbf{lower bound} on the black-box complexity of \textsc{BinVal}, which is proven for the first time and is equal to
          $\log_2 n + 1.1186406 - \Theta(\log_2 n / 2^n)$.
\end{itemize}

The rest of the paper is structured as follows. Section~\ref{sect:pre} introduces the necessary notation and definitions.
Section~\ref{sect:unrest} studies unrestricted and unbiased black-box complexities of \textsc{BinVal}. Finally,
Section~\ref{sect:conc} concludes and gives final remarks.

\section{Preliminaries}\label{sect:pre}

Throughout the paper, we denote the set of integer numbers $\{1, 2, \ldots, n\}$ as $[1..n]$.
We denote as $\hamm{x,y}$ the Hamming distance between two bit strings $x$ and $y$ of the same length,
that is, the number of positions at which they do not agree. In a few places we use the notation
$\tobinary{x}$ for the integer number $x$ written in the binary numeral system,
where the necessary number of leading zeros may need to be prepended depending on the context. 

We denote by $E_{\mathcal{A}}(f)$ the expected running time of an algorithm $\mathcal{A}$ on a function $f$,
that is, the expected number of queries to that function until its optimum is queried for the first time.
The \emph{black-box complexity} of a class of functions $F$ for a class of algorithms $\mathcal{X}$ is the following value~\cite{doerr2014reducing}:
\begin{align*}
    BBC_{\mathcal{X}}(F) = \inf_{\mathcal{A} \in \mathcal{X}} \sup_{f \in F} E_{\mathcal{A}}(f).
\end{align*}

Throughout the paper, we consider only the algorithms operating on bit strings of fixed length.
A \emph{$k$-ary variation operator}~\cite{doerr2014reducing}
$\theop$ produces a search point $y$ from the given $k$ search points $x_1, \ldots, x_k$ with probability $P_{\theop}(y \mid x_1, \ldots, x_k)$.
The operator $\theop$ is \emph{unbiased}~\cite{unbiased-bbc-algorithmica} if the following relations hold for all search points $x_1, \ldots, x_k, y, z$ and all permutations $\pi$ over $[1..n]$:
\begin{align*}
    P_{\theop}(y \mid x_1, \ldots, x_k) &= P_{\theop}(y \oplus z \mid x_1 \oplus z, \ldots, x_k \oplus z), \\
    P_{\theop}(y \mid x_1, \ldots, x_k) &= P_{\theop}(\pi(y) \mid \pi(x_1), \ldots, \pi(x_k)),
\end{align*}
where $a \oplus b$ is the bitwise exclusive-or operation applied to two bit strings $a$ and $b$ of the same length, and $\pi(a)$ is an application of permutation $\pi$ to a bit string $a$.

The algorithm is a \emph{$k$-ary unbiased black-box algorithm}~\cite{doerr2014reducing} if on every iteration it performs the following actions:
\begin{itemize}
    \item Based only on the fitness values of already queried individuals, it chooses:
    \begin{itemize}
        \item a non-negative integer number $k' \le k$;
        \item $k'$ individuals among the already queried ones, possibly with repetitions, noting that the order of the individuals matters, and
        \item a $k'$-ary unbiased variation operator.
    \end{itemize}
    \item It applies the chosen operator to the chosen individuals.
    \item Finally, it evaluates the fitness of an individual which the chosen operator produced.
\end{itemize}

Another definition for the unbiased black-box algorithm~\cite{generic-unbiased-algorithms} does not explicitly limit the set of possible variation operators,
but instead requires that the probability distributions of the produced individuals do not change when the problem undergoes certain transformations
(which, in the case of pseudo-Boolean problems, are exactly the class of transformations preserving Hamming distances). However,
the definition of a $k$-ary unbiased black-box algorithm becomes identical with the definition from~\cite{generic-unbiased-algorithms} when
$k$ is set to infinity, which is allowed by the particular flavour of the definition above. For this reason, any theorem, that holds for
$k$-ary unbiased black-box algorithms assuming arbitrary $k$, also holds for ``just'' unbiased algorithms.

Closely following~\cite{ranking-based-complexity}, we consider the following class of functions called ``binary value'', or \textsc{BinVal},
defined on bit strings of length $n$:
\begin{equation}
\textsc{BinVal}_{z,\pi}(x) = \sum_{i=1}^{n} 2^{i-1} \cdot [z_i = x_{\pi(i)}], \label{binval-basic}
\end{equation}
where $z \in \{0;1\}^n$ is the hidden bit string representing the unknown optimum, $\pi: [1..n] \to [1..n]$
is a hidden permutation of indices from the range $[1..n]$ that defines which weights are given to which indices,
and the \emph{Iverson bracket} $[.]$ is the notation for a function that converts the logic truth to~1 and the logic false to~0.
This function has a single global maximum at $x = z$, which we strive to find, with the corresponding function value of $2^n-1$.

Unlike the ``classical'' binary value function, where the weights are rigidly assigned to the bit indices,
such definition makes the problem class symmetric with regards to isomorphisms of the Hamming cube that represents the search space.
A similar procedure has been done on another famous benchmark function, \textsc{LeadingOnes}, to prove its black-box complexity
of $\Theta(n \log \log n)$~\cite{bbc-leadingones}.

One of the properties that make this function special is that it essentially defines a bijection between the queried bit strings and the
function values. This bijection also preserves the Hamming distance: for any $z$ and $\pi$, and for any two bit strings $x_1$ and $x_2$, it holds that
$\hamm{x_1, x_2} = \hamm{\tobinary{f(x_1)}, \tobinary{f(x_2)}}$, where $f = \textsc{BinVal}_{z,\pi}$.
As a result, this function exposes as much of the information about the hidden parameters as possible for an unbiased setting.
This property makes \textsc{BinVal} ``an easiest function'' regarding its unbiased black-box complexity
(and, in general, $k$-ary unbiased black-box complexity for any $k$).
More formally, the following lemma holds:

\begin{lemma}
For any integer $k \ge 0$ and $X$ being a class of $k$-ary unbiased black-box algorithms,
and for any class of pseudo-Boolean functions $F$ defined on bit strings of length $n$ which
(i) consists of functions with a single global optimum and 
(ii) is symmetric under the Hamming cube isomorphisms,
$BBC_{\mathcal{X}}(\textsc{BinVal}) \le BBC_{\mathcal{X}}(F)$.
\label{binval-easiest}
\end{lemma}
\begin{proof}
As the problem class $F$ is symmetric under the Hamming cube isomorphisms and consists of functions with a single global optimum,
there is a problem instance $f_0 \in F$ such that the all-ones bit string is a unique optimum.

Now we show that for any algorithm $\mathcal{A} \in \mathcal{X}$ 
it holds that $E_{\mathcal{A}}(g) \le E_{\mathcal{A}}(f_0)$. To do this, we note that the composite problem
$h$, defined as $h(x) = f_0(\tobinary{g(x)})$, belongs to $F$, because the mapping $x \mapsto \tobinary{f(x)}$ is 
a Hamming-preserving bijection. As $\mathcal{A}$ is unbiased, $E_{\mathcal{A}}(h) = E_{\mathcal{A}}(f_0)$.
On the other hand, an optimum of $h$ is necessarily an optimum of the problem instance $g$, as only such argument to $g$
yields an all-ones bit string, which is an optimum for $f_0$. This, in turn, means that
$E_{\mathcal{A}}(g) \le E_{\mathcal{A}}(h) = E_{\mathcal{A}}(f_0)$.

Now we apply the definitions of black-box complexities and derive that:
\begin{align*}
    BBC_{\mathcal{X}}(\textsc{BinVal}) &:= \inf_{\mathcal{A} \in \mathcal{X}} \sup_{g \in \textsc{BinVal}} E_{\mathcal{A}}(g)
     \\&\le \inf_{\mathcal{A} \in \mathcal{X}} \sup_{g \in \textsc{BinVal}} E_{\mathcal{A}}(x \mapsto f_0(\tobinary{g(x)})) = \inf_{\mathcal{A} \in \mathcal{X}} E_{\mathcal{A}}(f_0)
     \\&\overset{(*)}{=} \inf_{\mathcal{A} \in \mathcal{X}} \sup_{f \in F} E_{\mathcal{A}}(f)
       =: BBC_{\mathcal{X}}(F),
\end{align*}
where $:=$ and $=:$ signs denote applying the definition of the black-box complexity, and the move denoted by $(*)$
follows from the fact that $\mathcal{X}$ is a set of unbiased algorithms.
\end{proof}

The indefinite article in ``an easiest function'' reflects the fact that there are infinitely many function classes
which are equally easy compared to \textsc{BinVal}. In particular, the following class of functions is of interest at least within this paper:

\begin{lemma}
The following class of functions, defined on bit strings of length $n$ and parameterized by
\textbf{known} real-valued weights $\vec{w} = [w_1, \ldots, w_n]$ such that $w_{i+1} \ge 2 \cdot w_i$ for all $1 \le i < n$,
has the same $k$-ary unbiased black-box complexity, for any $k \ge 0$, as \textsc{BinVal}:
\begin{equation*}
\textsc{BinValEx}^{(\vec{w})}_{z,\pi}(x) = \sum_{i=1}^{n} w_{i} \cdot [z_i = x_{\pi(i)}].
\end{equation*}
\end{lemma}
\begin{proof}
The direction $BBC_{\mathcal{X}}(\textsc{BinValEx}) \ge BBC_{\mathcal{X}}(\textsc{BinVal})$,
for $\mathcal{X}$ being the class of $k$-ary unbiased black-box algorithms,
follows from Lemma~\ref{binval-easiest}. Now we prove the reverse direction by reducing a \textsc{BinValEx} problem to a \textsc{BinVal} problem.

It follows from the property of the weights, $w_{i+1} \ge 2 \cdot w_i$,
that $\sum_{j=1}^{i} w_j \le \sum_{j=1}^{i} 2^{j-i} \cdot w_i \le (2 - 2^{1-i}) \cdot w_i < 2 \cdot w_i$.
As a result, the following greedy algorithm can, given a fitness value $f$ on some search point $x$, tell at which
$w_i$ the bit $x_{\pi(i)}$ matches the bit $z_i$:
\begin{enumerate}
    \item $i \gets n$.
    \item if $f < w_i$, tell the bit for $w_i$ does not match, go to step (\ref{binvalex-greedy-next}). \label{binvalex-greedy-test}
    \item $f \gets f - w_i$, tell the bit for $w_i$ matches.
    \item $i \gets i - 1$, if $i > 0$ go to step (\ref{binvalex-greedy-test}). \label{binvalex-greedy-next}
\end{enumerate}

This means that we can easily recreate the value of \textsc{BinVal} for the same $z$ and $\pi$ from the value given by $\textsc{BinValEx}^{(\vec{w})}$
if the weights are known.
\end{proof}

In particular, a subproblem of \textsc{BinVal} on arbitrary chosen bit indices, for which the weights are known, is an instance of \textsc{BinValEx}
and can be solved in the same way as \textsc{BinVal} of the corresponding size. In the rest of the paper we treat \textsc{BinVal} and the derived
\textsc{BinValEx}-type problems uniformly and call them all \textsc{BinVal}.

\section{Unrestricted and Unbiased BBC of \textsc{BinVal}}\label{sect:unrest}

The main result of this paper is as follows:
\begin{theorem}
The unrestricted, as well as the unbiased black-box complexity of \textsc{BinVal} is 
at most $$\log_2 n + 2.42141558 - \Theta(\log_2 n / 2^n)$$
and at least $$\log_2 n + 1.1186406 - \Theta(\log_2 n / 2^n).$$
\label{unbiased-bbc-binval}
\end{theorem}

Informally, one can imagine an algorithm that works in $O(\log n)$. From the result $\textsc{BinVal}(x_0)$ of the initial query $x_0$,
which is essentially a random bit string, we can find the list of weights, at which the bits are guessed correctly, but not yet their positions.
By issuing the next query $x_1$ where a half of the bits is flipped (randomly, in the case of an unbiased algorithm, or arbitrarily if the algorithm is unrestricted)
and analyzing $\textsc{BinVal}(x_1)$, we get which weights correspond to bits which coincide in $x_0$ and $x_1$, and which correspond to differing bits,
but nothing more. These halves, or subproblems, are essentially two \textsc{BinValEx} functions, which we can treat as plain \textsc{BinVal} functions from now on.

The main idea is that we can optimize them in parallel by combining the queries coming from the subproblems into a single query to the original problem,
and on receiving an answer we can easily split it into answers to the queries of the subproblems. On a next step, each of these two problems is again subdivided into halves,
and this process continues until the subproblem sizes approach one. A small fraction of the subproblems will find their answers preliminarily by occasionally making
all bits equal one or zero. One can also slightly optimize the algorithm by taking an advantage of knowing the number of bits guessed right and deriving an optimal decision
on how to split the problem into two subproblems.

We formalize these ideas using the following set of statements.

\begin{definition}
$E(n, d)$, where $0 \le d \le n$, is the expected time to optimize a uniformly sampled problem of size $n$,
taken from the \textsc{BinVal} class, using an optimal algorithm,
given that only the first query is already made and the Hamming distance to the optimum is $d$.
\end{definition}

It is clear that $E(n, 0) = 0$, since the first query has already queried the optimum, and
$E(n, n) = 1$, as the optimum is the complete inverse of the first query. This also means that for $n = 1$
all possible values are already known. The following lemma helps to derive all other values.

\begin{lemma}
The following holds for $n > 1$ and $0 < d < n$:
\label{end-recursion}
\begin{align*}
E(n, d) &= 1 + \min_{\mathclap{0 < s < n}} E(n, d, s), \text{ where}\\
E(n, d, s) &= \sum_{\mathclap{t = \max(0, s+d-n)}}^{\mathclap{\min(s,d)}} 
    \max(E(s, s - t), E(n - s, d - t))
    \cdot\frac{\binom{s}{t}\binom{n-s}{d-t}}{\binom{n}{d}}.
\end{align*}
\end{lemma}

\begin{proof}
This expression corresponds to making a wise choice for $s$, which is the number of bits to flip randomly,
performing the query, which is reflected by ``$1+$'', and then by optimally optimizing the
appearing subproblems, taking into account that the $d$ bits which are not guessed right
can appear randomly in both of the subproblems. As mentioned above, both subproblems are instances of \textsc{BinValEx},
since by comparing the fitnesses of the first and the second queries we can determine which weights correspond to the
coinciding bits and to the differing bits. It is easy to see that in both subproblems all the available information
can be described as the fitness of the corresponding part of the second query, so we can refer to
$E(s, s-t)$ and $E(n-s, d-t)$ as expected times to solve these subproblems.

It is clear that $E(n, d)$ cannot be greater than the right-hand side,
as we definitely can do as well as described. Now we show that it cannot be smaller as well.

Assume that there is a smallest $n_0 > 1$ such that for some $d_0$ it holds that $E(n_0, d_0)$ is smaller than the right-hand side.
Consider the next query, and let $s_{\text{opt}}$ be the random variable denoting the amount of bits
to flip in the first query to construct the second query. It is clear that $P[s_{\text{opt}} = 0] = 0$
and $P[s_{\text{opt}} = d] = 0$ as well. Under an assumption that $E(n_0, d_0)$ is smaller than the right-hand side,
there exists an $s_0$ such that $0 < s_0 < n_0$, $P[s_{\text{opt}} = s_0] \neq 0$,
and the optimal algorithm is faster than $E(n_0, d_0, s_0)$.

Now consider the value of $t_s$, the random variable that describes the number of bits differing from the optimum in the first query among those bits which
are different in the first and the second query. It appears that $P[t_s = t_0] = \frac{\binom{s_0}{t_0}\binom{n_0-s_0}{d_0-t_0}}{\binom{n_0}{d_0}}$,
as whatever way the optimal algorithm selects the $s_0$ bits to flip,
the distribution of the instances of \textsc{BinVal} that still agree with the two queries is still uniform.
As a result, it means that there exists some value $t_0$ such that the runtime of the optimal algorithm is smaller
than $\max(E(s_0, s_0 - t_0), E(n_0 - s_0, d_0 - t_0))$. It follows that the maximum over these two values is greater
than the optimal expected time to solve the corresponding subproblem. As $s_0 < n_0$ and $n_0 - s_0 < n_0$,
this contradicts the assumption that $n_0$ is the smallest problem size at which the lemma statement brings suboptimal results.
This contradiction proves the lemma.
\end{proof}

Now we show that $E(n, d)$ are constrained quite well.

\begin{lemma}
For all $n > 0$ and all $0 < d < n$ it holds that:
\label{end-upper}
\begin{equation*}
    E(n, d) \le \log_2 n + 1 + \sum_{z = 0}^{\mathclap{\lceil \log_2 n \rceil - 1}} \log_2 \left(1 + \frac{1}{2^z}\right).
\end{equation*}
\end{lemma}

\begin{proof}
We prove this statement using induction by $n$.
For $n = 1$, this statement holds trivially.
Now we assume that for all $n' < n$ the induction statement holds.

For convenience, we introduce the following shorthand:
\begin{equation*}
    \phi(n) = 1 + \sum_{z = 0}^{\mathclap{\lceil \log_2 n \rceil - 1}} \log_2 \left(1 + \frac{1}{2^z}\right).
\end{equation*}
Note that $\phi(n)$ is a non-decreasing function, as well as $\log_2 n$ and their sum.

The following two inequalities hold for all $n$: $E(n, 0) = 0 \le \log_2 n + \phi(n)$ and $E(n, n) = 1 \le \log_2 n + \phi(n)$.
We simplify $E(n, d, s)$ from Lemma~\ref{end-recursion} as follows:
\begin{align*}
E(n, d, s) &= \sum_{\mathclap{t = \max(0, s+d-n)}}^{\mathclap{\min(s,d)}} 
    \max(E(s, s - t), E(n - s, d - t))
    \cdot\frac{\binom{s}{t}\binom{n-s}{d-t}}{\binom{n}{d}} \\
             &\le \sum_{\mathclap{t = \max(0, s+d-n)}}^{\mathclap{\min(s,d)}}
    (\log_2 \max(s, n - s) + \phi(\max(s, n - s)))
    \cdot\frac{\binom{s}{t}\binom{n-s}{d-t}}{\binom{n}{d}} \\
             &= \log_2 \max(s, n - s) + \phi(\max(s, n - s)).
\end{align*}

If $n$ is even, $\min_s \max(s, n - s) = \frac{n}{2}$ and the following holds:
\begin{equation*}
E(n, d) \le 1 + \log_2 \frac{n}{2} + \phi\left(\frac{n}{2}\right) \le \log_2 n + \phi(n).
\end{equation*}

If $n$ is odd, $\min_s \max(s, n - s) = \frac{n + 1}{2}$, and the derivation is slightly more complicated:
\begin{align*}
E(n, d) &\le 1 + \log_2 \frac{n+1}{2} + \phi\left(\frac{n+1}{2}\right) \\
          &= \log_2 n + \log_2 \left(1 + \frac{1}{n}\right) + \phi\left(\frac{n+1}{2}\right).
\end{align*}

Now we choose $k$ such that $2^k + 1 \le n < 2^{k+1}$. Note that the last inequality is strict as $n$ is odd.
In the following, we use the fact that $2^k + 2 \le n + 1 \le 2^{k+1}$, which results in:
\begin{equation*}
    2^{k-1} + 1 \le \frac{n + 1}{2} \le 2^{k}.
\end{equation*}

As a consequence, the sum in $\phi$ runs up to $k-1$. Now we use the following inequality:
\begin{equation*}
    \log_2 \left(1 + \frac{1}{n}\right) < \log_2 \left(1 + \frac{1}{2^k}\right)
\end{equation*}
and finish the analysis for the case of $n$ being odd:
\begin{align*}
E(n, d) &< \log_2 n + \log_2 \left(1 + \frac{1}{2^k}\right) + \phi\left(\frac{n+1}{2}\right)\\
          &= \log_2 n + \log_2 \left(1 + \frac{1}{2^k}\right) + 1 + \sum_{z = 0}^{k-1} \log_2 \left(1 + \frac{1}{2^z}\right)\\
          &= \log_2 n + 1 + \sum_{z = 0}^{k} \log_2 \left(1 + \frac{1}{2^z}\right)
           = \log_2 n + \phi(n). \qedhere
\end{align*}
\end{proof}

Note that the sum in the statement of Lemma~\ref{end-upper} is bounded from above by a constant:
\begin{equation*}
    \sum_{z=0}^{\infty} \log_2 \left(1 + \frac{1}{2^z}\right) = 2.2535240379347\ldots \le 2.26,
\end{equation*}
so a straightforward corollary from Lemma~\ref{end-upper} is that $E(n, d) \le \log_2 n + 3.26$.

In fact, we can refine this additive constant by evaluating all $E(n, d)$ by definition for all $n \le 2^{k}$,
computing the maximum difference $D_{\max} = E(n, d) - \log_2 n$, and adding the following expression to $D_{\max}$:
\begin{equation*}
    \sum_{z=k}^{\infty} \log_2 \left(1 + \frac{1}{2^z}\right).
\end{equation*}
For $k = 10$ we found that $D_{\max} < 1.4194631$, and the analytical remainder converges to a number slightly smaller
than $0.00195248$, which together proves that $E(n, d) \le \log_2 n + 1.42141558$.

We proceed with the lower bounds. 
\begin{lemma}
For all $n > 0$ and $0 < d < n$ it holds that:
\label{end-lower}
\begin{equation*}
E(n, d) \ge \log_2 n + 0.1186406.
\end{equation*}
\end{lemma}
\begin{proof}
For $n = 1$, this statement holds as the set of possible $d$ is empty.
We prove the following by induction for $n \ge 2$:
\begin{equation*}
    E(n, d) \ge \log_2 n + \xi(\lfloor \log_2 n \rfloor),
\end{equation*}
where $\xi(t)$ is defined as follows:
\begin{align*}
    \xi(1) &= \frac{1}{2}, &
    \xi(t) &= \xi(t - 1) - \frac{2 t + 2 \xi(t - 1) - 3}{\binom{2^t}{2^{t-1}}}.
\end{align*}
Note that $\xi(t)$ is strictly positive and decreases with $t$, however, the difference $\xi(t) - \xi(t+1)$ decreases sharply as $t$ grows.
Another useful property\footnote{The proof is moved to Appendix~\ref{apx-magic-3} for clarity.} is that $f(n) = \log_2 n + \xi(\lfloor \log_2 n \rfloor)$ is a non-decreasing function,
from which it follows that $\max(f(n_1), f(n_2)) = f(\max(n_1, n_2))$.

The base of the induction is for $n = 2$, where $E(2, 1) = 1.5 = \log_2 2 + \xi(1)$, and for $n = 3$,
where $E(3, 1) = E(3, 2) = \frac{7}{3} > 2.0849626 > \log_2 3 + \xi(1)$.
The following needs to be proven for $n \ge 4$, assuming for all smaller $n$ the lemma statement holds.
Consider $E(n, d, s)$ from Lemma~\ref{end-recursion}. Let $t_{\min} = \max(0, s + d - n)$, $t_{\max} = \min(s, d)$, then:
\begin{align*}
    E(n, d, s) &= \sum_{t = t_{\min}}^{t_{\max}} \max(E(s, s - t), E(n - s, d - t)) \cdot P[t] \\
               &= \sum_{t = t_{\min} + 1}^{t_{\max} - 1} \max(E(s, s - t), E(n - s, d - t)) \cdot P[t] \\
               &+ E_{\min}(n, d, s) + E_{\max}(n, d, s), \\
               &\ge (\log_2 \max(s, n - s) + C) \cdot (1 - P[t_{\min}] - P[t_{\max}]) \\
               &+ E_{\min}(n, d, s) + E_{\max}(n, d, s),
\end{align*}
where the following notation is used:
\begin{align*}
 C    &= \xi(\lfloor \log_2 \max(s, n - s) \rfloor) \ge \xi(\lfloor \log_2 n \rfloor), &
 P[t] &= \tbinom{s}{t}\tbinom{n - s}{d - t} / \tbinom{n}{d},
\end{align*}
\vspace{-4.5ex}
\begin{align*}
 E_{\min}(n, d, s) &= \max(E(s, s - t_{\min}), E(n - s, d - t_{\min})) \cdot P[t_{\min}], \\
 E_{\max}(n, d, s) &= \max(E(s, s - t_{\max}), E(n - s, d - t_{\max})) \cdot P[t_{\max}].
\end{align*}

The reduction of the middle of the sum is possible as for $t \in (t_{\min}; t_{\max})$ the second
arguments of both $E(s, s - t)$ and $E(n - s, d - t)$ never turn extreme. For $t = t_{\min}$
and $t = t_{\max}$ the situation is different: for both of them, at least one of the arguments under the maximum
becomes extreme (which means either 0 for $E(n, 0)$ or 1 for $E(n, n)$). What exactly may happen, depends on how $n$, $s$ and $d$ are related.

\textbf{Case 1:} $s = d$. This renders $t_{\max} = s = d$ and $E_{\max}(n, s, d) = 0$ as both values under $\max$ have their seconds arguments zeroed out.
The corresponding probability $P[t_{\max}]$ becomes $1 / \binom{n}{d}$.

\textbf{Case 2:} $s + d = n$. This renders $t_{\min} = 0$ and $E_{\min}(n, s, d) = 1$ as both values under $\max$ have their second arguments maximized.
The corresponding probability $P[t_{\min}]$ also becomes $1 / \binom{n}{d}$.

It appears fruitful to consider two situations: the two cases above may happen together or separately.

\textbf{Both 1 and 2:} 

In this case, $s = d$ and $s + d = n$, so $s = d = n/2$ and $n$ is even. Then it holds that (assuming $w = \lfloor \log_2 n \rfloor$):
\begin{align*}
E(n, d, s) &= \left(\log_2 \frac{n}{2} + C\right)\cdot\left(1 - \frac{2}{\binom{n}{n/2}}\right) + \frac{1}{\binom{n}{n/2}} \\
           &= \log_2 \frac{n}{2} + \xi\left(\left\lfloor \log_2 \frac{n}{2} \right\rfloor\right) - \frac{2 \log_2 n+ 2\xi(\lfloor \log_2 \frac{n}{2} \rfloor) - 3}{\binom{n}{n/2}} \\
           &\ge \log_2 \frac{n}{2} + \xi(w - 1) - \frac{2w + 2\xi(w-1) - 3}{\binom{2^w}{2^{w-1}}} \\
           &= \log_2 n - 1 + \xi(w) = \log_2 n - 1 + \xi(\lfloor \log_2 n \rfloor).
\end{align*}

\textbf{Either 1 or 2:} The cases are almost symmetrical, except that Case 1 makes the corresponding subexpression zero and Case 2 makes it one, so we consider
Case 1 as a worse one. In this case, the following holds:
\begin{equation*}
    E(n, d, s) \ge (\log_2 \max(d, n - d) + C) \cdot \left(1 - \frac{1}{\binom{n}{d}}\right),
\end{equation*}
and we are going to show that it is at most $\log_2 (n/2) + C$. As it cannot happen that $d = n/2$ (otherwise Case 2 happens as well), we are free to consider $1 \le d < n/2$,
as $n/2 < d < n$ is symmetric.

Now we prove that left hand side of the inequality decreases as $2 \le d < n/2$. We consider the following function for $n \ge 5$:
\begin{equation*}
\psi(n, d, C) = (\log_2 (n - d) + C) \cdot \left(1 - \frac{1}{\binom{n}{d}}\right).
\end{equation*}

After algebraic transformations\footnote{The proof is moved to Appendix~\ref{apx-magic-1} for clarity.} we find that:
\begin{align*}
\psi(n, d, C) - \psi(n, d+1, C) &= \log_2\left(1 + \frac{1}{n - d - 1}\right) \cdot \left(1 - \frac{1}{\binom{n}{d}}\right)  \\
                                &- \frac{\log_2 (n - d - 1) + C}{\binom{n}{d}}\cdot\left(1 - \frac{d+1}{n-d}\right).
\end{align*}

Note that the difference grows as $d$ grows, so it is enough to prove that it is non-negative when $d = 2$. After such substitution,
replacing $\log_2$ with $\ln$ and multiplying $C$ by $\ln 2$,
and using that $\ln(1 + x) \ge x - x^2/2$, we get that the difference is at most:
\begin{equation*}
\left(\frac{1}{n-3} - \frac{1}{2(n-3)^2}\right)\left(1 - \frac{2}{n(n-1)}\right) - 2\cdot\frac{\ln(n-3) + C \ln 2}{n(n-1)} \cdot\frac{n-5}{n-2}.
\end{equation*}

By multiplying by $(n-3)$, which retains the sign, opening all brackets and retaining only negative addends apart from $+1$, we get the following:
\begin{equation*}
1 - \frac{1}{2(n-3)} - \frac{2}{n(n-1)} - 2\cdot\frac{\ln (n-3) + C \ln 2}{n}\cdot\frac{(n-3)(n-5)}{(n-1)(n-2)}.
\end{equation*}

By replacing the last fraction by 1, applying $\frac{\ln(n-3)}{n} \le 0.202$ for $n \ge 5$ and replacing all $n$ with 5, we get:
\begin{equation*}
1 - \frac{1}{4} - \frac{1}{10} - 0.404 - \frac{2 \ln 2 \cdot C}{5} = 0.246 - \frac{2 \ln 2 \cdot C}{5},
\end{equation*}
which is positive for all $C \le 0.88$.
For this reason, the lower bound can be reached on either $\psi(n, 1, C)$ or $\psi(n, \lfloor \frac{n-1}{2} \rfloor, C)$.

First, we show that:
\begin{align*}
\psi(n, 1, C) - \log_2 \frac{n}{2} + C &= (\log_2 (n-1) + C) \cdot \left(1 - \frac{1}{n}\right) - \log_2 \frac{n}{2} - C\\
              &= \log_2 \left(2 - \frac{2}{n}\right) - \frac{\log_2 (n-1) + C}{n}.
\end{align*}

This is non-negative for $n \ge 4$, using that $\log_2 (n-1) / n \le 0.402$
and the rest of the function is monotone in $n$, for all $C \le 0.728$.

Next, we prove similar relations for $\psi(n, \lfloor \frac{n-1}{2} \rfloor, C)$, which are written uniformly
using $\Delta = 2$ for even $n$ and $\Delta = 1$ for odd $n$:
\begin{equation*}
\psi(n, \frac{n-\Delta}{2}, C) - (\log_2 \frac{n}{2} + C) = \log_2\left(1 + \frac{\Delta}{n}\right) - \frac{\log_2 \frac{n+\Delta}{2} + C}{\binom{n}{\frac{n-\Delta}{2}}} \ge 0.
\end{equation*}

For this we again use $\ln(1 + x) \ge x - x^2/2$ and prove\footnote{The proof is moved to Appendix~\ref{apx-magic-2} for clarity.} that the last subtrahend, divided by $n$, grows with $n$.
The proofs work for $C \le 0.57908006$.

As a result, when either Case 1 or Case 2 happens, but not both, we proved that $E(n, d, s) \ge \log_2 \frac{n}{2} + C = \log_2 n + C - 1$.

\textbf{Case 3:} all other cases. In these cases, only one of the values under $\max$ becomes either 0 or 1, while another value obeys the general rule
and is at least one. There are two pairs of symmetrical cases, and without loss of generality we consider only one of them: $t_{\min} = 0, n > s + d$.
This case influences $E_{\min}$, which becomes:
\begin{equation*}
E_{\min}(n, s, d) = E(n - s, d) \cdot P[0] \ge (\log_2 (n - s) + C) \cdot P[0].
\end{equation*}

If $n - s \ge s$, $E_{\min}$ is not actually altered and follows the general scheme, so we further assume that $s > n - s$. Under this assumption,
no other symmetric cases from this point can simultaneously influence $E(n, d, s)$, so we write:
\begin{align*}
    E(n, d, s) &\ge C + P[0] \log_2 (n - s) + (1 - P[0]) \log_2 s \\
               &=   C + \log_2 s - P[0] \cdot (\log_2 s - \log_2 (n - s)) \\
               &=   C + \log_2 s - \frac{\binom{n - s}{d}}{\binom{n}{d}} \cdot (\log_2 s - \log_2 (n - s)) \\
               &\ge C + \log_2 s - \frac{n - s}{n} \cdot (\log_2 s - \log_2 (n - s)) \\
               &=   C + \frac{n - s}{n} \log_2 (n - s) + \frac{s}{n} \log_2 s \\
               &\ge C + \log_2 \frac{n}{2} = C + \log_2 n - 1.
\end{align*}

As a summary, in all cases except when Case 1 and Case 2 happen simultaneously, 
$E(n, d, s) \ge \log_2 n + C - 1 \ge \log_2 n + \xi(\lfloor \log_2 n \rfloor) - 1$,
and in the remaining case $E(n, d, s) \ge \log_2 n + \xi(\lfloor \log_2 n \rfloor) - 1$ was proven directly,
which completes the induction.

To complete the proof of the entire lemma, it is enough to note that $\xi(n)$ converges to its limit with the superpolynomial speed.
For instance, $\xi(5) \approx 0.11864060660016391$ and $\xi(6) = \xi(5)$ within the double floating point precision.
As a result, one can safely use the lower bound $\xi(n) \ge 0.1186406$.
\end{proof}

In practice, due to various pessimizations in our proofs, the lower bound is a little bit better: $E(n, d) \ge \log_2 + \frac{1}{6}$
for all $n$ and $0 < d < n$. Now we are ready to prove the main result of this paper.

\begin{proof}[Proof of Theorem~\ref{unbiased-bbc-binval}]
Lemmas~\ref{end-upper} and~\ref{end-lower}, along with the refinement of the former based on the exact computations of $E(n,d)$,
impose the following restrictions on $E(n, d)$ for $0 < d < n$:
\begin{equation*}
    \log_2 n + 0.1186406 \le E(n, d) \le \log_2 n + 1.42141558.
\end{equation*}

The black-box complexity is:
\begin{equation*}
BBC(\textsc{BinVal}) = 1+\sum_{d=0}^{n} E(n,d) \frac{\binom{n}{d}}{2^n},
\end{equation*}
which we bound from above as follows:
\begin{align*}
    BBC(\textsc{BinVal}) &\le 1 + \frac{1}{2^n} + \sum_{d = 1}^{n-1} (\log_2 n + 1.42141558) \frac{\binom{n}{d}}{2^n} \\
                         &= 1 + \frac{1}{2^n} + (\log_2 n + 1.42141558) \cdot \left(1 - \frac{2}{2^n}\right) \\
                         &= \log_2 n + 2.42141558 - \Theta\left(\frac{\log n}{2^n}\right),
\end{align*}
and from below by $\log_2 n + 1.1186406 - \Theta(\log n / 2^n)$ in the same way.
\end{proof}

\section{Conclusion}\label{sect:conc}

\begin{figure}[!t]
\begin{tikzpicture}
\begin{axis}[width=\textwidth, xmode=log, legend pos=north west]
    \addplot table[x=n, y=bbc] {plots.dat};
    \addlegendentry{BBC};
    \addplot table[x=n, y=lower] {plots.dat};
    \addlegendentry{Lower bound (no $o(1)$ term)};
    \addplot table[x=n, y=upper] {plots.dat};
    \addlegendentry{Upper bound (no $o(1)$ term)};
    \addplot table[x=n, y=ceil] {plots.dat};
    \addlegendentry{Upper bound from~\cite{ranking-based-complexity}};
\end{axis}
\end{tikzpicture}
\caption{Plots of the exact black-box complexity of \textsc{BinVal} and of its upper and lower bounds}\label{plot}
\end{figure}
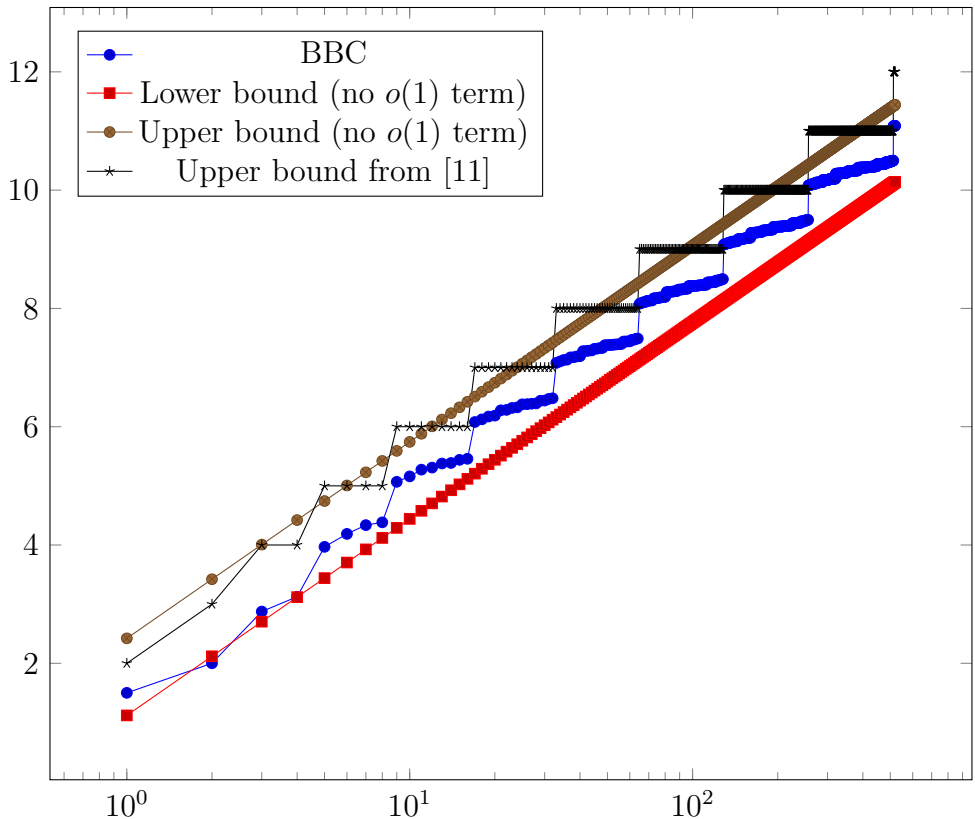

We proved quite sharp bounds on the black-box complexity of permutation-enabled version of the binary value (\textsc{BinVal}) function,
which is $\log_2 n + \Theta(1)$, where the constants that define $\Theta(1)$ are also known and their difference is less than $1.4$.
The upper bound complements the existing upper bound from~\cite{ranking-based-complexity}, as it is more precise for roughly a half of
problem sizes (see Fig.~\ref{plot} for visual comparison), and the lower bound was proven for the first time.

We feel that \textsc{BinVal}, due to its all-revealing fitness, may be used as a convenient tool to prove the lower bounds
on $k$-ary unbiased black-box complexities of unimodal functions, including \textsc{OneMax}, for which it is still an open question
for many years.

\bibliographystyle{abbrv}
\bibliography{../../../../bibliography}

\newpage
\appendix

\section{Properties of $\xi$}\label{apx-magic-3}

The function $\xi(k)$ is defined as follows:
\begin{align*}
    \xi(1) &= \frac{1}{2}, \\
    \xi(k+1) &= \xi(k) - \frac{2 k + 2 \xi(k) - 1}{\binom{2^{k+1}}{2^{k}}}.
\end{align*}

\textbf{First} we show that $\xi(k+1) \ge \xi(k)/3$. This will automatically show that $\xi(k) > 0$ and that it decreases as $k$ grows.
From this fact, it also holds that there exists a constant $C_0 \ge 0$ such that $\xi(k) \ge C_0$.

To do this, we use induction to assume $\xi(k) \ge \frac{1}{2} \cdot \frac{1}{3^{k-1}}$ and show that:
\begin{align*}
    \xi(k) \cdot \frac{2}{3} - \frac{2k + \xi(k) - 1}{\binom{2^{k+1}}{2^k}}
  &= \xi(k) \cdot \left(\frac{2}{3} - \frac{2}{\binom{2^{k+1}}{2^k}}\right) - \frac{2k-1}{\binom{2^{k+1}}{2^k}}
\\&\ge \frac{1}{2} \cdot \frac{1}{3^{k-1}} \cdot \left(\frac{2}{3} - \frac{2}{\binom{4}{2}}\right) - \frac{2k-1}{\binom{2^{k+1}}{2^k}}
\\&= \frac{1}{3^k} \cdot \left( \frac{1}{2} - \frac{(2k-1) \cdot 3^k}{\binom{2^{k+1}}{2^k}} \right).
\end{align*}

To get out the last occurences of $k$ that may influence the sign of the expression, we show that the last fraction is non-increasing:
\begin{align*}
\frac{ \frac{(2k+1) \cdot 3^{k+1}}{\binom{2^{k+2}}{2^{k+1}}} }
     { \frac{(2k-1) \cdot 3^k}{\binom{2^{k+1}}{2^k}} }
  &= 3 \cdot \left(1 + \frac{2}{2k-1}\right) \cdot \frac{2^{k+1}! 2^{k+1}! 2^{k+1}!}{2^{k+2}! 2^k! 2^k!}
\\&= 3 \cdot \left(1 + \frac{2}{2k-1}\right) \cdot \frac{(2^{k+1})^2}{2^{k+2}(2^{k+2}-1)} \\&\cdot \frac{(2^{k+1}-1)^2}{(2^{k+2}-2)(2^{k+2}-3)} \cdot \ldots \cdot \frac{(2^k+1)^2}{(2^{k+1}+2)(2^{k+1}+1)}
\\&= 3 \cdot \left(1 + \frac{2}{2k-1}\right) \cdot \frac{1}{2^{2^k}} \cdot \frac{2^{k+1}}{2^{k+2}-1} \cdot \frac{2^{k+1}-1}{2^{k+2}-3} \cdot \ldots \cdot \frac{2^k+1}{2^{k+1}+1}
\\&\le \frac{9}{4} \cdot \frac{2^{k+1}}{2^{k+2}-1} \cdot \frac{2^k+1}{2^{k+1}+1} \le \frac{27}{35} < 1.
\end{align*}

So we can finish proving this statement by substituting $k = 1$ into the last fraction:
\begin{equation*}
  \xi(k) \cdot \frac{2}{3} - \frac{2k + \xi(k) - 1}{\binom{2^{k+1}}{2^k}} 
    \ge \frac{1}{3^k} \cdot \left( \frac{1}{2} - \frac{(2k-1) \cdot 3^k}{\binom{2^{k+1}}{2^k}} \right)
    \ge \frac{1}{3^k} \cdot \left( \frac{1}{2} - \frac{3}{\binom{4}{2}} \right) \ge 0.
\end{equation*}

\textbf{Second} we show that $f(n) = \log_2 n + \xi(\lfloor \log_2 n \rfloor)$ is non-decreasing for integer $n \ge 2$. It is enough to show it for pairs $n_1 = 2^k-1$ and $n_2 = 2^k$, $k \ge 2$,
since on other occasions the $\xi$-related part does not change, and $\log_2 n$ is a non-decreasing function. First we transform the required inequality into a simpler form:
\begin{align}
    \log_2 (2^k - 1) + \xi(k - 1) &\overset{?}{\le} \log_2 2^k + \xi(k) \notag\\
    \log_2 (2^k - 1) + \xi(k - 1) &\overset{?}{\le} \log_2 (2^k - 1) + \log_2 \left(1 + \frac{1}{2^k-1}\right) + \xi(k - 1) \notag\\&- \frac{2 (k - 1)+ 2 \xi(k - 1) - 1}{\binom{2^k}{2^{k-1}}} \notag\\
    \frac{2 k + 2 \xi(k - 1) - 3}{\binom{2^k}{2^{k-1}}} &\overset{?}{\le} \log_2 \left(1 + \frac{1}{2^k-1}\right) \notag\\
    \frac{2 k + 2 \xi(k - 1) - 3}{\binom{2^k}{2^{k-1}}} &\overset{?}{\le} \frac{1}{\ln 2} \left(\frac{1}{2^k-1} - \frac{1}{2(2^k-1)^2}\right) \notag\\
    \frac{2 k + 2 \xi(k - 1) - 3}{2^k \cdot \frac{(2^k-2)!}{2^{k-1}!2^{k-1}!}} &\overset{?}{\le} \frac{1}{\ln 2} \left(1 - \frac{1}{2(2^k-1)}\right). \label{lastq}
\end{align}

Now we show that $\frac{k}{2^k \cdot \frac{(2^k-2)!}{2^{k-1}!2^{k-1}!}}$ is a decreasing function for $k \ge 2$. Indeed:
\begin{align*}
\frac{ \frac{k+1}{2^{k+1} \cdot \frac{(2^{k+1}-2)!}{2^{k}!2^{k}!}} }
     { \frac{k}{2^k \cdot \frac{(2^k-2)!}{2^{k-1}!2^{k-1}!}} }
     &= \frac{k+1}{2k} \cdot \frac{(2^k-2)! 2^k! 2^k!}{(2^{k+1}-2)! 2^{k-1}! 2^{k-1}!}
   \\&= \frac{k+1}{2k} \cdot \frac{(2^k)^2}{(2^{k+1}-2)(2^{k+1}-3)} \\&\cdot \frac{(2^k - 1)^2}{(2^{k+1}-4)(2^{k+1}-5)} \cdot \ldots \cdot \frac{(2^{k-1}+1)^2}{2^k (2^k-1)}
   \\&\le \frac{k+1}{2k} < 1.
\end{align*}

Now in~\eqref{lastq} there is a decreasing function on the left hand side and an increasing one on the right hand side. When $k = 2$, this is:
\begin{equation*}
    \frac{2 k + 2 \xi(k - 1) - 3}{2^k \cdot \frac{(2^k-2)!}{2^{k-1}!2^{k-1}!}} \le \frac{4 + 2 \cdot 0.5 - 3}{4 \cdot \frac{2!}{2!\cdot2!}} = 1
\le 1.202 < \frac{1}{\ln2} \left(1 - \frac{1}{6}\right),
\end{equation*}

which completes the proof.

\section{Tedious Transformation \#1}\label{apx-magic-1}

We consider the following function for $1 \le d < n/2$:
\begin{equation*}
\psi(n, d, c) = (\log_2 (n - d) + c) \cdot \left(1 - \frac{1}{\binom{n}{d}}\right).
\end{equation*}

The following holds:
\begin{align*}
\psi(n&, d, c) - \psi(n, d - 1, c) = \log_2 (n - d) + c - \frac{\log_2 (n - d) + c}{\binom{n}{d}} \\
		&- \log_2 (n - d - 1) - c + \frac{\log_2 (n - d - 1) + c}{\binom{n}{d+1}}  \\
              &= \log_2\left(1 + \frac{1}{n - d - 1}\right) - \frac{\log_2 (n - d) + c}{\binom{n}{d}} \\&+ \frac{\log_2 (n - d - 1) + c}{\binom{n}{d}}\cdot\frac{d+1}{n-d} \\
              &= \log_2\left(1 + \frac{1}{n - d - 1}\right) - \frac{\log_2 (n - d - 1) + \log_2\left(1 + \frac{1}{n - d - 1}\right) + c}{\binom{n}{d}} \\
                                                             &+ \frac{\log_2 (n - d - 1) + c}{\binom{n}{d}}\cdot\frac{d+1}{n-d}  \\
              &= \log_2\left(1 + \frac{1}{n - d - 1}\right) \cdot \left(1 - \frac{1}{\binom{n}{d}}\right) \\
                                                             &- \frac{\log_2 (n - d - 1) + c}{\binom{n}{d}}\cdot\left(1 - \frac{d+1}{n-d}\right).
\end{align*}

\newpage
\section{Tedious Transformation \#2}\label{apx-magic-2}

Using the same definition as above, we prove that $\psi(n, \lfloor \frac{n-1}{2} \rfloor, c) \ge \log_2 \frac{n}{2} + c$ \textbf{for small enough $c$}, and for $n \ge 4$.
To denote $\lfloor \frac{n-1}{2} \rfloor$ more conveniently, we use a macro $\Delta = 2$ for even $n$ and $\Delta = 1$ for odd $n$,
such that this value becomes $\frac{n-\Delta}{2}$. We use the lower bound for $\ln(1 + x) \ge x - x^2/2$.

\begin{align*}
\psi\left(n, \frac{n-\Delta}{2}, c\right) &- \left(\log_2 \frac{n}{2} + c\right)
\\&= \left(\log_2 \frac{n + \Delta}{2} + c\right) \cdot \left(1 - \frac{1}{\binom{n}{\frac{n-\Delta}{2}}}\right) - \left(\log_2 \frac{n}{2} + c\right)
\\&= \log_2\left(1 + \frac{\Delta}{n}\right) - \frac{\log_2 \frac{n+\Delta}{2} + c}{\binom{n}{\frac{n-\Delta}{2}}}
\\&= \frac{1}{\ln2}\cdot\left(\ln\left(1 + \frac{\Delta}{n}\right) - \frac{\ln \frac{n+\Delta}{2} + c\ln2}{\binom{n}{\frac{n-\Delta}{2}}}\right)
\\&\ge \frac{1}{\ln2}\cdot\left( \frac{\Delta}{n} - \frac{\Delta^2}{2n^2} - \frac{\ln \frac{n+\Delta}{2} + c\ln2}{\binom{n}{\frac{n-\Delta}{2}}}\right)
\\&=   \frac{\Delta}{n\ln2} \cdot\left(1 - \frac{\Delta}{2n} - \frac{\ln \frac{n+\Delta}{2} + c\ln2}{\Delta\cdot\frac{(n-1)!}{\frac{n-\Delta}{2}!\frac{n+\Delta}{2}!}}\right).
\end{align*}

First we want to show that the last expression from above grows with $n$ (increasing by 2) starting from $n = 3$ for $\Delta = 1$ and from $n = 4$ for $\Delta = 2$. For this, it is enough to show
that the following function grows with $n$, while $n$ increases each time by 2, starting from $2 + \Delta$:
\begin{equation*}
\eta(n, \Delta) = \frac{\ln \frac{n+\Delta}{2}}{\frac{(n-1)!}{\frac{n-\Delta}{2}!\frac{n+\Delta}{2}!}} =
\frac{(\ln \frac{n+\Delta}{2}) \frac{n-\Delta}{2}! \frac{n+\Delta}{2}!}{(n - 1)!}.
\end{equation*}

To do that, we estimate $\eta(n + 2, \Delta) / \eta(n, \Delta)$ and prove that it is at most one:
\begin{align*}
\frac{\eta(n + 2, \Delta)}{\eta(n, \Delta)} &= \frac{(\ln \frac{n+2+\Delta}{2}) \frac{n+2-\Delta}{2}! \frac{n+2+\Delta}{2}! (n-1)!}
                                                    {(\ln \frac{n+\Delta}{2}) \frac{n-\Delta}{2}! \frac{n+\Delta}{2}! (n+1)!}
\\&= \log_{\frac{n+\Delta}{2}}\left(\frac{n+2+\Delta}{2}\right) \cdot \frac{\frac{n+2-\Delta}{2} \cdot \frac{n+2+\Delta}{2}}{n(n+1)}
\\&= \log_{\frac{n+\Delta}{2}}\left(\frac{n+\Delta}{2}+1\right) \cdot \frac{(n+2-\Delta)\cdot(n+2+\Delta)}{4n(n+1)}
\\&\overset{(*)}{=} \log_{\frac{n+\Delta}{2}}\left(\frac{n+\Delta}{2}+1\right) \cdot \frac{(n+2+\Delta)}{4(n-1+\Delta)}
\\&= \log_{\frac{n+\Delta}{2}}\left(\frac{n+\Delta}{2}+1\right) \cdot \frac{1}{4}\cdot\left(1 + \frac{3}{n-1+\Delta}\right)
\\&\le \log_{\frac{n+\Delta}{2}}\left(\frac{n+\Delta}{2}+1\right) \cdot \frac{1}{4}\cdot\left(1 + \frac{3}{2\Delta+1}\right)
\\&\le \log_{\frac{n+\Delta}{2}}\left(\frac{n+\Delta}{2}+1\right) \cdot \frac{1}{2} < 1,
\end{align*}
where in the move labeled ($*$) we considered what happens when $\Delta=1$ and $\Delta=2$ and ``optimized'' the appearance of the result,
and the last move used the fact that $x^2 > x + 1$ for $x \ge 2$.

Next, we estimate $\eta(n, \Delta)$ from above by taking it at $n = 6 - \Delta$ (which is $n = 4$ for $\Delta = 2$ and $n = 5$ for $\Delta = 1$):
\begin{align*}
\eta(n, \Delta) &\le \eta(6 - \Delta, \Delta) = \frac{(\ln \frac{6}{2}) \frac{6 - 2\Delta}{2}! \frac{6}{2}!}{(5 - \Delta)!} 
\\&= \frac{6\ln 3 \cdot (3-\Delta)!}{(5-\Delta)!} = \frac{6 \ln 3}{(4-\Delta)(5-\Delta)}.
\end{align*}

Now the statement that we need to prove is essentially equivalent to proving the following:
\begin{align*}
1 - \frac{\Delta}{2n} - \frac{\ln \frac{n+\Delta}{2} + c\ln2}{\Delta\cdot\frac{(n-1)!}{\frac{n-\Delta}{2}!\frac{n+\Delta}{2}!}} \ge 
1 - \frac{\Delta}{12 - 2\Delta} - \frac{6 (\ln 3 + c \ln 2)}{\Delta(4-\Delta)(5-\Delta)} \ge 0.
\end{align*}

For $\Delta = 1$ this holds for $c \le \frac{9 - 5 \ln 3}{9 \ln 2} \approx 1.011888$, for $\Delta = 2$ it holds for $c \le \frac{1.5 - \ln 3}{\ln 2} \approx 0.57908006$.

\end{document}